\documentclass[twoside]{article}

\usepackage[preprint]{aistats2027}
\usepackage[round]{natbib}

\usepackage{tikz}
\usetikzlibrary{arrows.meta,positioning, tikzmark, calc}
\usepackage{algorithm}
\usepackage{booktabs}
\usepackage{algpseudocode}   
\usepackage{amsmath}
\usepackage{amssymb}
\usepackage{amsthm}
\theoremstyle{plain}      

\usepackage{soul}

\newtheorem{theorem}{Theorem}
\newtheorem{lemma}[theorem]{Lemma}
\newtheorem{proposition}[theorem]{Proposition}

\theoremstyle{definition} 
\newtheorem{definition}[theorem]{Definition}

\theoremstyle{remark}

\usepackage[colorlinks=true, citecolor=blue, linkcolor=blue, urlcolor=blue]{hyperref}
\usepackage{graphicx}  
\usepackage{booktabs}  
\usepackage{bm}
\usepackage{xcolor} 
\usepackage{soul}

\DeclareMathOperator*{\argmin}{arg\,min}
\DeclareMathOperator*{\argmax}{arg\,max}

\begin{document}

%

%

\twocolumn[

\aistatstitle{Learning Non-Gaussian SEMs with Conditional Risk, and Why the Direction of Recursion Matters}

\aistatsauthor{Ujas Shah \And Panagiota Birmpa \And  Rebecca Morrison}

\aistatsaddress{University of Colorado Boulder \And  Heriot-Watt University \And University of Colorado Boulder} 
]

\begin{abstract}
Variance-based methods recover a Gaussian linear structural equation model (SEM) with equal or suitably constrained error variances, either by recursively identifying the source nodes with smallest conditional variance (top-down method) or sink nodes with highest conditional variance (bottom-up method). We show that the use of conditional variance is incidental: replacing it with the \emph{conditional risk}, i.e., the minimal expected loss of linear prediction under the negative log-likelihood of the noise distribution, extends both methods to non-Gaussian noise. Existing methods are the least-squares special case. This yields two main results. First, exploiting the noise distribution improves structure recovery: the matched-loss procedures recover the SEM more accurately than their least-squares counterparts and the LiNGAM methods, which are designed for non-Gaussianity. Second, the two methods, interchangeable in the Gaussian setting, separate under non-Gaussian noise: the top-down method is correct for every noise density with finite differential entropy, whereas the bottom-up method requires a regularity condition and can fail without it, for instance under mean-zero exponential or Cauchy noise. Simulations confirm both results.
\end{abstract}

\section{Introduction}
Directed acyclic graphs (DAGs), or Bayesian networks, provide a compact representation of the conditional independence and causal relations among a collection of random variables \citep{spirtes2000causation, Pearl_2009}. Given a random vector \(X = (X_1, \dots, X_d)^\top\), its joint distribution can be associated with a DAG \(\mathcal{G}\) whose nodes correspond to the variables and directed edges encode the dependence of each variable on its parents. Equivalently, the vector can be described by a structural equation model (SEM) as a collection of equations of the form: \[X_i = g_i \big(X_{\mathrm{PA}(i)}, N_i \big), \qquad i = 1, \dots, d, \label{eq:sem_general}\] where \(\mathrm{PA}(i) \subset \{1, \dots, d\} \setminus \{i\}\) indexes the parents of \(X_i\), and \(N_i\) is an exogenous noise variable, i.e., a random variable whose generating mechanism is not specified by the causal model. Recovering the DAG, or the SEM, that generated the set of observations uncovers the dependence structure among the variables, and, under the causal Markov condition and the assumption of causal sufficiency, carries causal information~\citep{eberhardt2017introduction}. The SEM identifies which variables directly influence which others, and how the system would respond to interventions (understood as external modifications that fix the value of a variable and propagate through the causal structure to its descendants). Consequently, structure learning has found applications in varied fields such as bioinformatics~\citep{sachs2005causal, shen2020challenges}, ecology~\citep{lefcheck2016piecewisesem, runge2019inferring}, and finance~\citep{sanford2012bayesian, hyvarinen2010estimation}.

\paragraph{Learning the graph structure.} Learning the graph from purely observational data is difficult, for two reasons. The first is computational: the number of DAGs grows super-exponentially in the number of nodes, and searching over this space is NP-hard in general~\citep{chickering1996learning, chickering2004large}. The second is statistical: without further assumptions, the DAG is identifiable only up to its Markov equivalence class (the set of graphs encoding the same conditional independence relations). A substantial literature has developed methods that recover this equivalence class, most prominently the constraint-based PC algorithm~\citep{spirtes1991algorithm}, the score-based greedy equivalence search (GES)~\citep{chickering2002optimal}, and, more recently, continuous-optimization formulations~\citep{deng2024markov}.

In this paper, we are interested in recovering the true DAG itself, rather than its Markov equivalence class. We focus on the widely studied linear additive noise SEM, in which each variable is a linear function of its parents plus independent noise:
\begin{equation}
    X = \bm \Lambda X + N,
    \label{eq:sem_linear}
\end{equation}
where $\bm \Lambda \in \mathbb{R}^{d \times d}$ encodes the linear dependence of each variable on its parents and $N = (N_1, \dots, N_d)^\top$ is a vector of mutually independent noise variables. We assume that the system has no latent confounders. Additional assumptions, such as non-Gaussianity or equal error variances, can make DAGs identifiable from data~\citep{shimizu2006linear, loh2014high, peters2014identifiability}. Under equal error variances, simple variance-based algorithms can recover the topological order of the DAG by recursively identifying either source nodes, i.e., nodes without parents, with the smallest conditional variance (top-down method; \citealp{chen2019causal}), or sink nodes, i.e., nodes without children, with the highest conditional variance (bottom-up method; \citealp{ghoshal2018learning}). These variance-based methods are distribution free and make no use of knowledge of the error distribution. In many applications, however, one may have reasonable knowledge of the error distribution---for example, heavy-tailed errors in finance---and it is natural to ask whether this knowledge can be leveraged for more accurate structure recovery.

\paragraph{Our results.} In this work we answer this question affirmatively and uncover an unexpected asymmetry between the two methods (top-down and bottom-up) along the way. We extend the top-down and bottom-up variance-based approaches by replacing their use of conditional variance with a \emph{conditional risk}: the minimal expected loss of linear prediction under the negative log-likelihood of the assumed noise density. The Gaussian equal-variance methods are recovered as the least-squares case, and their recursive order-identification logic transfers to our methods. Under least squares, the top-down and bottom-up recursions are equally valid; under general losses, they are not. The top-down recursion, which only ever regresses a variable on its non-descendants, is correct for any noise density of finite differential entropy; the bottom-up recursion additionally requires the noise to satisfy a regularity condition. We show examples of noise distributions that do not meet the condition, where the bottom-up method fails, but the top-down method recovers the graph. In this sense, our work is both a bridge between the Gaussian and non-Gaussian structure learning literature and a re-examination of the choice between top-down and bottom-up recursion. Simulation experiments confirm this: losses matched to the noise perform better than their least-squares counterparts and LiNGAM family methods, and the top-down method dominates the bottom-up method, mildly in the regular regime and decisively in degenerate ones.

\paragraph{Paper organization.} The remainder of the paper is organized as follows. Section~\ref{sec:prelim_and_related} reviews related work, Section~\ref{sec:notation} fixes notation, Section~\ref{sec:method} presents the method and theory, and Section~\ref{sec:experiments} the experiments.

\section{Related Work}
\label{sec:prelim_and_related}
Literature on structure learning for the linear additive noise model~\eqref{eq:sem_linear} divides naturally according to whether the noise is Gaussian or non-Gaussian. We organize related work by noise distribution (Figure~\ref{fig:method-grid}).

\begin{figure*}[t]
\centering
\def\labw{2cm}
\def\cellw{6.5cm}
\def\txtw{6.0cm}   
\def\hlw{6cm}    
\begin{tikzpicture}[
  font=\small,
  head/.style={draw, rounded corners=2pt, fill=gray!15, font=\small\bfseries,
    align=center, anchor=north west, inner sep=3pt},
  cell/.style={draw, rounded corners=2pt, fill=gray!4, anchor=north west,
    minimum width=\cellw},
  txt/.style={anchor=north west, align=left, text width=\txtw, inner sep=0pt},
  hl/.style={draw, rounded corners=2pt, fill=gray!20, anchor=north west,
    align=left, text width=\hlw, inner sep=3pt},
]
  \def\xA{2.5}   
  \def\xB{9.30}   
  \def\yH{0}      
  \def\yG{-1.0}  
  \def\yN{-3.5}  

  \node[head, minimum width=\cellw, minimum height=0.7cm, text width=6.0cm]
    at (\xA, \yH) {Arbitrary noise scales};
  \node[head, minimum width=\cellw, minimum height=0.7cm, text width=6.0cm]
    at (\xB, \yH) {Identically distributed noise};

  \node[head, minimum width=\labw, minimum height=2.30cm, text width=2cm]
    at (0, \yG) {Gaussian\\errors};
  \node[head, minimum width=\labw, minimum height=2.30cm, text width=2cm]
    at (0, \yN) {Non-Gaussian\\errors};

  \node[cell, minimum height=2.3cm] at (\xA, \yG) {};
  \node[txt] at (\xA+0.15, \yG-0.15)
    {\textendash~Identifiable only up to the Markov \\
    \phantom{\textendash}~equivalence class (e.g., PC, GES)};

  \node[cell, minimum height=2.30cm] at (\xB, \yG) {};
  \node[txt] (gtxt) at (\xB+0.15, \yG-0.15)
    {\textendash~Greedy DAG search\\
     \textendash~Continuous optimization (NOTEARS,\\
     \phantom{\textendash}~GOLEM, DAGMA)\\};
  \node[hl] (var) at ([yshift=-1pt]gtxt.south west)
    {\textendash~Top-down and bottom-up conditional\\
    \phantom{\textendash}~variance methods};
    
  \node[cell, minimum height=2.30cm] at (\xA, \yN) {};
  \node[txt] at (\xA+0.15, \yN-0.15)
    {\textendash~LiNGAM family (ICA-, Direct-,\\
     \phantom{\textendash~}Opt-LiNGAM)\\
     \textendash~Continuous optimization (NOTIME)};

  \node[cell, minimum height=2.30cm] at (\xB, \yN) {};
  \node[txt] (ntxt) at (\xB+0.15, \yN-0.15)
    {\textendash~Methods from the cell above and the \\
    \phantom{\textendash}~ cell to the left apply\\
     \textendash~Entropy-based ordering (ETPIA)};
  \node[hl] (ours) at ([yshift=-1pt]ntxt.south west)
    {\textendash~\textbf{Top-down and bottom-up}\\
    \phantom{\textendash}~\textbf{conditional risk methods (ours)}};

  \draw[-{Latex[length=2mm]}, dashed]
    (ours.east) -- ++(0.45,0) |- (var.east)
    node[pos=0.25, rotate=90, anchor=north, align=center, font=\footnotesize]
      {Special case under \\
      Gaussian noise};
\end{tikzpicture}
\caption{Representative classes of methods that recover the true DAG of a linear SEM, by noise distribution (rows) and noise-scale assumption (columns). The list is not exhaustive. The bottom-right assumptions are special cases of both neighboring cells, so their methods also apply there, provided the required moments exist.}
\label{fig:method-grid}
\end{figure*}

When the noise is Gaussian, the DAG is identifiable only up to Markov equivalence~\citep{Pearl_2009}. \citet{peters2014identifiability} showed that identifiability of the \emph{true} DAG is restored under the assumption of equal error variances, $\sigma_1^2 = \dots = \sigma_d^2 = \sigma^2$: no two distinct DAGs with equal-variance Gaussian noise can generate the same observational distribution. 

This motivated variance-based algorithms that learn a topological ordering one node at a time, then select parents and estimate edge weights. Top-down methods~\citep{chen2019causal} recursively identify a source node among the remaining vertices: under equal variances, a source of the remaining subgraph attains the minimal conditional variance given previously selected vertices. Bottom-up approaches~\citep{ghoshal2018learning} instead recursively identify a sink node: under equal variance, a sink shows the highest conditional variance given all other variables. These early variance-based guarantees hinge on accurate estimation of the covariance or precision matrix, and their sample complexity consequently degrades with the conditioning of the covariance matrix $\Sigma$: the required number of samples grows polynomially with the condition number, a dependence that can be severe since the condition number may itself grow with the depth of the graph. More recent work removes this dependence by replacing global precision-matrix estimation with nodewise regressions that directly estimate the relevant conditional variances \citep{park2023computationally,pmlr-v267-daskalakis25a}.

The equal-variance assumption can be relaxed somewhat. Identifiability continues to hold when the error variances are not equal but known, and variance-based recovery extends to heteroscedastic settings in which the noise variances satisfy suitable ratio or ordering conditions \citep{ghoshal2018learning, park2020heterogeneous}. Our analysis carries over to these settings, but for clarity we focus on the equal-variance case here.

A complementary line of work casts structure learning as a continuous optimization problem by encoding acyclicity through a smooth function of the adjacency matrix, such as the trace exponential of NOTEARS~\citep{zheng2018dags}, or the log-determinant characterization of DAGMA~\citep{bello2022dagma}. Another method, GOLEM~\citep{ng2020role}, replaces the hard constraint with soft sparsity and acyclicity penalties. Because these scores are built on squared loss or the Gaussian likelihood, their minimizers are guaranteed to recover the true DAG only under equal (or known) noise variances, the same assumption that underlies variance-based methods~\citep{reisach2021beware, ng2024structure}.

When the errors are non-Gaussian, the linear SEM is identifiable without any assumption on the noise scales by the LiNGAM (Linear non-Gaussian acyclic model) family of methods. ICA-LiNGAM~\citep{shimizu2006linear} estimates the ordering with independent component analysis, while DirectLiNGAM~\citep{shimizu2011directlingam} recovers topological order top-down, selecting at each step the variable that is most independent of the residuals of the remaining variables regressed on it. \citet{hyvarinen2013pairwise} replace the independence measure in DirectLiNGAM with pairwise likelihood ratios. More recently, \citet{oh2025optimal} introduced OptLiNGAM, a sample-optimal procedure for SEMs with sub-Gaussian noise, and NOTIME~\citep{pmlr-v258-berrevoets25a} brought identifiability under LiNGAM to differentiable structure learning by penalizing the joint dependence of the fitted residuals.

Closest to our work are methods that order variables by entropy. \citet{xie2019efficient} study linear SEMs with i.i.d.\ noise (ETPIA). They approximate entropy with a Gram--Charlier expansion in the variance, skewness, and kurtosis, and then remove the source's effect by least-squares regression during recursion, so the method requires finite fourth moments. \citet{feigenbaum2024causal} identify sources and sinks by comparing conditional entropies with noise entropy, but their study is limited to discrete noise distributions. Our criterion is also entropy-based, but focuses on continuous variables in linear models.

Other methods keep the recursive ordering but replace the variance with another measure of residual uncertainty. \citet{gao2020polynomial} (NPVAR) order nonparametric DAGs top-down using residual variances from nonparametric regression, without assuming linearity or independent noise. \citet{rajendran2021structure} replace the conditional variance with Bregman information, a generalization of variance induced by an exponential family, and prove identifiability when the Bregman information is equal across variables. Our conditional risk differs from both. While \citet{gao2020polynomial} measure uncertainty by variance alone, our criterion assumes a known noise density and uses its full shape. Unlike \citet{rajendran2021structure}, we do not require an exponential family and cover densities such as Student-$t$ and Cauchy. We also study both top-down and bottom-up recursions, while both of these works use only the top-down direction.

Our work sits between the Gaussian and non-Gaussian regimes. Like the variance-based methods, it learns the ordering by a recursion over nodewise regressions; like the LiNGAM family, it identifies the DAG under non-Gaussian noise, including heavy-tailed noise with no finite variance. The price is the constraint on the scale of noise variables.

\section{Notation}
\label{sec:notation}
For a positive integer $d$, let $[d] = \{1, \dots, d\}$. We consider a directed acyclic graph $\mathcal{G} = (V, E)$ with vertex set $V = [d]$, where each vertex $i$ is associated with a random variable $X_i$, and an edge set $E \subset V \times V$, where $(j,i) \in E$ denotes the directed edge $j \to i$. For a vertex $i$, we write $\mathrm{PA}(i)$, $\mathrm{CH}(i)$, $\mathrm{AN}(i)$, and $\mathrm{DE}(i)$ for its sets of parents, children, ancestors, and descendants in $\mathcal{G}$, respectively. The non-descendants of \(i\) are $V \setminus (\mathrm{DE}(i) \cup \{i\})$, i.e., the vertices other than $i$ that cannot be reached from $i$ by a directed path. A vertex with no parents is a source (or root) and a vertex with no children is a sink (or terminal vertex); every DAG contains at least one of each. A permutation \(t=(t_1,\ldots,t_d)\) of $[d]$ is a topological ordering of $\mathcal{G}$ if every parent precedes its child: writing \(t^{-1}(i)\) for the position of vertex \(i\), $t^{-1}(j) < t^{-1}(i)$ whenever $j \in \mathrm{PA}(i)$. Acyclicity guarantees that at least one such ordering exists, though it need not be unique.

We study the linear structural equation model~\eqref{eq:sem_linear} where $X = (X_1, \dots, X_d)^\top$, the matrix $\bm \Lambda \in \mathbb{R}^{d \times d}$ has $\Lambda_{ij} \neq 0$ if and only if $j \in \mathrm{PA}(i)$, and $N = (N_1, \dots, N_d)^\top$ is a vector of i.i.d.\ noise variables with location zero. Acyclicity of $\mathcal{G}$ is equivalent to the existence of a permutation matrix $\bm P$ such that $\bm P \bm \Lambda \bm P^\top$ is strictly lower triangular. Writing $\bm B = \bm I - \bm \Lambda$, which is invertible under acyclicity, the model implies $X = \bm B^{-1} N$; when the noise has variance \(\sigma^2\) the precision matrix of $X$ is $\bm \Omega = \bm \Sigma^{-1} = \bm B^\top \bm D^{-1} \bm B$, where $\bm D = \sigma^2 \bm I$. The goal of structure learning is to recover $\mathcal{G}$ (equivalently, the support of $\bm \Lambda$) from the samples of \(X\).

\section{Generalized Ordering-Based Learning via Conditional Risks}
\label{sec:method}

The variance-based methods reviewed in Section~\ref{sec:prelim_and_related} all rest, in one form or another, on the fact that under the equal-variance assumption, conditional variances reveal the topological order. In this section we first recall this fact and then show that neither direction actually depends on the conditional variance. Replacing squared error with the objective matched to the (known) noise distribution yields a family of \emph{conditional risks} that support the same recursive strategies, in both directions, while using the full noise density rather than only second moments.

\subsection{Conditional Variance and Least Squares}
\label{ssec:condvar_lsq}

For a subset $S \subset [d]$ and $i \notin S$, the linear conditional variance of $X_i$ given $X_S$ admits the variational form
\begin{equation}
    \mathrm{Var_{lin}}(X_i \mid X_S) \;=\; \min_{\beta \in \mathbb{R}^{|S|}} \; \mathbb{E}\big(X_i - \beta^\top X_S\big)^2.
    \label{eq:condvar_variational}
\end{equation}
For jointly Gaussian variables this statistic coincides with the conditional variance, but in general is only the \emph{linear} conditional variance. Under equal noise variances, this statistic helps us recognize variable ordering from both ends. From the top: if $S$ is an ancestral set, i.e., closed under taking parents, a remaining variable attains $\mathrm{Var_{lin}}(X_i \mid X_S) = \sigma^2$ if all its parents lie in $S$, and exceeds $\sigma^2$ otherwise; recursively selecting the minimizer is the top-down method. From the bottom, the relevant fact is the following restatement of Proposition~2 from~\citet{ghoshal2018learning}.

\begin{lemma}[Sinks maximize the full-conditional variance]
\label{lem:ls_sink}
Consider the model~\eqref{eq:sem_linear} with mutually independent noises of common variance $\sigma^2$. For every $i \in [d]$,
\begin{equation}
    \mathrm{Var_{lin}}\big(X_i \mid X_{-i}\big) \;=\; \frac{1}{\Omega_{ii}} \;=\; \frac{\sigma^2}{1 + \sum_{j \in \mathrm{CH}(i)} \Lambda_{ji}^2},
    \label{eq:cond-variance}
\end{equation}
so that $\mathrm{Var_{lin}}(X_i \mid X_{-i}) = \sigma^2$ if $i$ is a sink and $\mathrm{Var_{lin}}(X_i \mid X_{-i}) \le \sigma^2 / (1 + b_{\min}^2) < \sigma^2$ otherwise, where $b_{\min} = \min_{(i,j) \in E} |\Lambda_{ji}|$.
\end{lemma}

For proofs, see Appendix~\ref{app:proofs}. This fact uses just two properties of the variance functional: (i) a regression whose covariates are non-descendants, and whose covariate set contains the parents, leaves a pure-noise residual, and no linear predictor can do better; and (ii) when a variable is regressed on a set containing its children, the children carry information about its noise, and the objective registers this by dropping strictly below the pure-noise value. Property (i) generalizes to any density with finite entropy, and (ii) generalizes only under a regularity condition.

\subsection{Conditional Risks}
\label{ssec:cond-risk}

Throughout this section, the noise variables $N_1, \dots, N_d$ are assumed independent and identically distributed, each with known density $f$; this \emph{equal-distribution} assumption is the natural generalization of the equal-variance assumption of~\citet{peters2014identifiability}, to which it reduces when $f$ is Gaussian. Let $\rho_f = -\log f$ denote the associated negative log-likelihood loss. This lets us inform the conditional risk objective~\eqref{eq:cond-risk} with our knowledge of the noise density, while the conditional variance objective~\eqref{eq:condvar_variational} remains only second-moment informed. 

\begin{definition}[Conditional risk]
\label{def:cond_risk}
For $S \subset [d]$ and $i \notin S$, the conditional risk of $X_i$ given $X_S$ is
\begin{equation}
    \mathcal{R}_f \big(X_i \mid X_S\big) \;=\; \min_{\beta \in \mathbb{R}^{|S|}} \; \mathbb{E} \left[\rho_f\big(X_i - \beta^\top X_S\big) \right],
    \label{eq:cond-risk}
\end{equation}
with the convention $\mathcal{R}_f(X_i \mid X_\emptyset) = \mathbb{E}\left[\rho_f(X_i) \right]$.
\end{definition}

In each case, $\rho_f$ is the loss under which linear regression is maximum-likelihood estimation for the assumed noise.

\begin{table}[ht]
\centering
\small
\begin{tabular}{@{}ll@{}}
\toprule
\textbf{Noise distribution} & $\rho_f(u)=-\log f(u)$\\
\midrule
Standard Normal
  & $\displaystyle \frac{u^2}{2} + C$ \\[3pt]
Mean-zero Laplace with scale $s$
  & $\displaystyle \frac{|u|}{s} + \log(2s)$ \\[3pt]
Student-$t_\nu$
  & $\displaystyle \frac{\nu+1}{2}
      \log\!\left(1+\frac{u^2}{\nu}\right) + C$ \\[5pt]
Mean-zero unit exponential
  & $\displaystyle
    \begin{cases}
      u+1, & u\geq -1,\\
      +\infty, & u<-1
    \end{cases}$ \\
\bottomrule
\end{tabular}
\caption{Noise distributions and their corresponding matched losses.}
\label{tab:matched_losses}
\end{table}

Define the \emph{pure-noise risk}
\begin{equation}
    c_f \;=\; \mathbb{E} \left[ \rho_f(N_i) \right] \;=\; -\mathbb{E} \left[ \log f(N_i) \right] \;=\; h(f),
    \label{eq:pure_noise_risk}
\end{equation}
which is the differential entropy of the noise, assumed finite throughout. This constant plays the role that $\sigma^2$ plays in the classical setting: it is the risk value that certifies a residual as being pure noise.

\subsection{The Top-Down Direction}
\label{ssec:top-down}

The top-down generalization relies on a risk decomposition that applies whenever a variable is regressed on its non-descendants and makes explicit the role of correctly specifying \(f\). 

\begin{lemma}[Source-node identification]
\label{lem:certificate}
Consider the model \eqref{eq:sem_linear} with i.i.d.\ noise of density $f$, and let $\rho_f = -\log f$. Let $i \in [d]$ and let $S$ consist of non-descendants of $i$. Then for any $\beta$, there exists $Z_{i, \beta, S}$  independent of $N_i$ such that $X_i - \beta^\top X_S = N_i + Z_{i, \beta, S}$. We write \(Z_\beta\) as short for \(Z_{i, \beta, S}\) with the \(i \text{ and } S\) being clear from context, and \(g_\beta\) as the density of \(N_i + Z_\beta\). Finally, 
\begin{equation}
\begin{aligned}
    \mathcal{R}_f \big(X_i \mid X_S\big) 
    &= \min_{\beta\in\mathbb R^{|S|}} \mathbb{E}\left [\rho_f \big(X_i - \beta^\top X_S\big) \right] \\
    &= \min_{\beta\in\mathbb R^{|S|}} \left\{ h\big(g_\beta\big) + \mathrm{KL}\big(g_\beta \,\big\|\, f\big) \right\} \\
    &\ge h(f) \\
    &= c_f,
\end{aligned}
\label{eq:risk_decomposition}
\end{equation}
with equality if and only if $Z_{\beta} = 0$. Consequently: if $\mathrm{PA}(i) \subseteq S$, then $\mathcal{R}_f(X_i \mid X_S) = c_f$, attained exactly at the coefficient vector supported on $\mathrm{PA}(i)$ with the true weights; and if $\mathrm{PA}(i) \not\subseteq S$ and $S$ is ancestral, then $\mathcal{R}_f(X_i \mid X_S) > c_f$.
\end{lemma}

Hence, \(c_f\) is the lower bound as long as one regresses \(X_i\) only on non-descendants. The proof needs only \( |h(f)| < \infty\) as mentioned in~\eqref{eq:pure_noise_risk}, no moments of \(f\), no smoothness or no support conditions. The top-down procedure now follows cleanly: grow an ancestral set by repeatedly selecting a remaining variable of minimal risk given the current set.

\begin{theorem}[Population correctness, top-down]
\label{thm:top-down}
Under the model \eqref{eq:sem_linear} with i.i.d.\ noise of density $f$ with $|h(f)| < \infty$, initialize $S_1 = \emptyset$ and for $k = 1, \dots, d$ select
\begin{equation}
    t_k \;\in\; \argmin_{i \notin S_k} \; \mathcal{R}_f\big(X_i \mid X_{S_k}\big),
    \qquad S_{k+1} = S_k \cup \{t_k\}.
\end{equation}
Then $(t_1, \dots, t_d)$ is a topological ordering of $\mathcal{G}$, and at each step the minimizing coefficient vector for the selected $t_k$ is supported exactly on $\mathrm{PA}(t_k)$, with values equal to the corresponding entries of $\Lambda$.
\end{theorem}

\subsection{The Bottom-Up Direction: Regularity Enters}
\label{sec:bottom-up}

The bottom-up procedure constructs the topological ordering in reverse by repeatedly selecting, among the remaining variables, the one with maximal full-conditional risk and removing it. The key distinction from the top-down setting is that, for a sink, all remaining variables are non-descendants, so Lemma~\ref{lem:certificate} applies and its full-conditional risk equals the baseline noise risk $c_f$. For a non-sink, however, the covariate set contains at least one child, which carries information about the node's own noise and can therefore reduce its conditional risk below $c_f$.  Establishing this strict reduction requires additional regularity of the noise density $f$, in contrast to the top-down argument, which relies only on finite differential entropy. Once this separation between sinks and non-sinks is established, maximizing the full-conditional risk identifies a sink at each iteration and yields a topological ordering in reverse.

\begin{lemma}[Stability under sink removal]
\label{lem:sink-removal}
Let $i$ be a sink of $\mathcal{G}$. The joint distribution of $X_{-i} = (X_j)_{j \neq i}$ is generated by the linear SEM whose graph is $\mathcal{G}$ with $i$ removed and whose noises are the original $(N_j)_{j \neq i}$, still i.i.d.\ with density $f$.
\end{lemma}

For a sink $i$, all other variables are non-descendants, so Lemma~\ref{lem:certificate} applies with $S = [d] \setminus \{i\}$ and gives $\mathcal{R}_f(X_i \mid X_{-i}) = c_f$. Moreover, for every vertex $j$, the coefficient vector placing the true parental weights on $\mathrm{PA}(j)$ and zero elsewhere is feasible and produces the residual $N_j$, so all nodes satisfy
\begin{equation}
    \mathcal{R}_f\big(X_j \mid X_{-j}\big) \;\le\; c_f
    \qquad \text{for every } j.
    \label{eq:free-upper-bound}
\end{equation}
As sinks satisfy the equality, they are always among the maximizers of the full-conditional risk. 

The bottom-up procedure needs the non-sinks to have a full conditional risk strictly below $c_f$ for the order to be identifiable. The covariate set of a non-sink contains its children, whose noise terms are not independent of the residual, so the decomposition \eqref{eq:risk_decomposition} no longer applies. Intuitively, a child is a noisy measurement of the parent $X_j$~'s noise, and regressing on it should explain part of $N_j$ away;  the question is: does using the child as a predictor actually decrease the expected $\rho$-risk? Unlike the top-down method, the answer requires something of the density. We impose the following regularity:

\begin{itemize}
    \item[(R1)] $f$ is absolutely continuous with a.e.\ derivative $f'$, $\mathbb{E}|N| < \infty$ for $N \sim f$, the score $\psi_f = \rho_f' = -f'/f$ satisfies $\mathbb{E}\left[ \psi_f(N_i)\right] = 0 $ and $\mathbb{E}\big[N_i\psi_f(N_i)\big] = 1$, and for independent $N_i, N_j \sim f$ and $b \neq 0$ the map $\varphi_f(\varepsilon) = \mathcal{R}_f\big(X_i \mid X_{-i}; \varepsilon\big) =  \mathbb{E}\left[ \rho_f\big((1 - \varepsilon b) N_i - \varepsilon N_j\big) \right]$ is differentiable at $\varepsilon = 0$ with derivative obtained by differentiating under the expectation.
\end{itemize}

The moment identities follow by integration by parts whenever $uf(u) \to 0$ as $|u| \to \infty$: $\int \psi_f f = -\int f' = 0$ and $\int u\,\psi_f(u) f(u)\,du = -\int u f'(u)\,du = \int f = 1$. The differentiation clause holds by dominated convergence if $\psi_f$ is a.e.\ continuous with $|\psi_f(u)| \le C(1+|u|)^k$ and $\mathbb{E}|N|^{k+1} < \infty$; this covers Gaussian ($k=1$), Laplace, and Student-$t_\nu$ with $\nu > 1$ ($k=0$).

\begin{lemma}[Strict risk reduction at non-sinks]
\label{lem:strict-gap}
Assume (R1). If $j$ is not a sink of $\mathcal{G}$, then
\begin{equation}
    \mathcal{R}_f\big(X_j \mid X_{-j}\big) \;<\; c_f.
\end{equation}
\end{lemma}

The existence of the non-sink conditional risk less than \(c_f\) is distribution-free once it satisfies the regularity condition, with the noise distribution entering only for the conditional risk gap between the sink and the non-sink. Assume $\varphi_f$ is twice differentiable at $0$ with $\varphi_f''(0) = \kappa_b = b^2\, \mathbb{E}\big[\psi_f'(N_i) N_i^2\big] + \sigma^2\, \mathbb{E}\,\big[\psi_f'(N_i)\big]$, where \(b\) is the edge weight connecting \(j\) to its child. Then approximating the non-sink's conditional risk with the Taylor expansion and subtracting from the sink's conditional risk (\(c_f\)) yields the gap
\begin{equation}
    c_f \;-\; \mathcal{R}_f\big(X_j \mid X_{-j}\big) \;\ge\; \frac{b^2}{2 \kappa_b} \;+\; o(b^2),
    \label{eq:quantitative-gap}
\end{equation}
which allows us to recover the correct ordering empirically even with sample errors. As a sanity check, for standard Gaussian noise $\psi_f'(u) \equiv 1$, so $\kappa_b = 1 + b^2$ and \eqref{eq:quantitative-gap} reads $b^2/(2(1+b^2))$---exactly the gap delivered by Lemma~\ref{lem:ls_sink}.

\begin{theorem}[Population correctness, bottom-up]
\label{thm:bottom-up}
Under the model \eqref{eq:sem_linear} with i.i.d.\ noise of density $f$ satisfying (R1), initialize $V_1 = [d]$ and for $k = 1, \dots, d$ select
\begin{equation}
    t_k \;\in\; \argmax_{i \in V_k} \; \mathcal{R}_f\big(X_i \mid X_{V_k \setminus \{i\}}\big),
    \qquad V_{k+1} = V_k \setminus \{t_k\}.
\end{equation}
Then $(t_d, \dots, t_1)$ is a topological ordering of $\mathcal{G}$, and at each step the minimizing coefficient vector for the selected $t_k$ is supported exactly on $\mathrm{PA}(t_k)$, with values equal to the corresponding entries of $\Lambda$.
\end{theorem}
Neither theorem requires faithfulness. Path cancellations never affect the arguments. As with the equal-variance methods, identifiability comes from the noise assumption, not from conditional independences.

\subsection{A Separation Between the Two Directions}
\label{ssec:separation}

The two generalizations are not equally robust. This distinction is obscured in the Gaussian setting, in which the top-down and bottom-up methods have primarily been studied and compared in prior work. The two methods identify an ordering by asking fundamentally different questions. The top-down procedure asks whether a variable can be explained by its non-descendants so that the remaining residual is pure noise. This is a goodness-of-fit question, and the entropy argument of Lemma~\ref{lem:certificate} answers it for every noise density of finite differential entropy. The bottom-up procedure instead asks whether a variable's noise can be partially predicted from its children. This is an estimation question: a non-sink must attain risk strictly below the pure-noise value \(c_f\). Whether the loss registers such a reduction depends on the noise distribution, which is why the bottom-up result requires the additional regularity condition (R1).

The distinction becomes sharp when (R1) fails. For mean-zero exponential noise, every vertex has full-conditional risk \(c_f\), so the bottom-up statistic contains no ordering information, whereas the top-down statistic continues to separate sources from non-sources. For standard Cauchy noise, the bottom-up statistic cannot distinguish the non-sink in the two-node model when \(0<|b|\leq1\), while the top-down statistic identifies the source for every \(b\neq0\). These failures arise, respectively, from information concentrated at a support boundary and from the additive scale geometry of the Cauchy distribution; complete arguments are given in Appendix~\ref{app:failure}. These are limitations of the full-conditional risk statistic, and hence of any ordering procedure built on it; they do not preclude recovery by methods outside this family (for instance, independence-testing approaches in the LiNGAM family, which exploit non-Gaussianity differently).

Further, we find that empirically the top-down method consistently recovers the DAG more accurately than the bottom-up method, as shown in Section~\ref{sec:experiments}. The two procedures schedule their most difficult ordering regressions differently. When \(m\) vertices remain, top-down methods condition on the \(d-m\) previously selected vertices, whereas bottom-up methods condition each candidate on the other \(m-1\) remaining vertices. Bottom-up therefore begins with \(d-1\) predictors while comparing all \(d\) candidates. Top-down begins with marginal risks and reaches its largest conditioning sets only after most of the ordering has already been determined. In the least-squares case, a regression with \(q\) predictors has approximately \(n-q\) residual degrees of freedom, making the early bottom-up comparisons noisier when \(n\) is not large relative to \(d\). An early error can also affect many subsequent recursion steps, whereas a late error has fewer opportunities to propagate. This is a finite-sample mechanism rather than a population-level separation, and a formal analysis is left for future work.

\subsection{Algorithms}
\label{ssec:algorithm}

The empirical procedures replace population risks with their empirical counterparts: for a variable $x_i$ and covariate set $\bm x_A$, the empirical conditional risk is
\begin{align*}
\widehat{\mathcal{R}}_{f}(x_i \mid \bm x_A) = \min_{\beta} \frac{1}{n}\sum_{k=1}^{n}
  \rho_f\big(x_i^{(k)} - \beta^{\top} \bm x_A^{(k)}\big).
\end{align*}

\begin{algorithm}[ht]
\caption{Risk-based top-down DAG recovery}
\label{alg:td}
\begin{algorithmic}[1]
\Require samples \(x^{(1)},\dots,x^{(n)}\); objective \(\rho_f\); threshold \(\omega>0\)
\State \(\widehat{ \bm \Lambda}\gets 0_{d\times d}\)
\State \(S\gets\emptyset\)
\For{\(k=1,\dots,d\)}
    \State \(t_k\gets\arg \min_{i\notin S}\widehat{\mathcal{R}}_f(x_i\mid \bm x_S)\)
    \ForAll{$j \in S$}
        \State $\widehat{\Lambda}_{t_k j} \gets \widehat{\beta}^{(t_k \mid S)}_j$
\EndFor
\EndFor
\State $\widehat{\Lambda}_{ab} \gets 0$ for all $(a,b)$ with $\bigl|\widehat{\Lambda}_{ab}\bigr| \le \omega$
\State \Return \(\widehat{\bm \Lambda}\) and the ordering \( (t_1, \cdots, t_d) \)
\end{algorithmic}
\end{algorithm}

\begin{algorithm}[ht]
\caption{Risk-based bottom-up DAG recovery}
\label{alg:bu}
\begin{algorithmic}[1]
\Require samples \(x^{(1)},\dots,x^{(n)}\); objective \(\rho_f\); threshold \(\omega>0\)
\State \(\widehat{ \bm \Lambda}\gets 0_{d\times d}\)
\State \(V\gets[d]\)
\For{\(k = 1, \cdots, d\)}
    \State
        \(t_k\gets\arg \max_{i\in V}
        \widehat{\mathcal{R}}_f(x_i\mid \bm x_{V\setminus\{i\}})\)
        \ForAll{\(j\in V\setminus\{t_k\}\)}
                \State $\widehat{\Lambda}_{t_k j} \gets \widehat{\beta}^{(t_k \mid V\setminus\{t_k\})}_j$
        \EndFor
        \State \(V\gets V\setminus\{t_k\}\)
\EndFor
\State $\widehat{\Lambda}_{ab} \gets 0$ for all $(a,b)$ with $\bigl|\widehat{\Lambda}_{ab}\bigr| \le \omega$
\State \Return \(\widehat{\bm \Lambda}\) and the ordering \((t_d, \cdots, t_1 )\)
\end{algorithmic}
\end{algorithm}

\section{Experiments}
\label{sec:experiments}

\paragraph{Setup.} We simulate Erd\H{o}s--R\'enyi DAGs (ER-\(k\))~\citep{erdos1960evolution} and scale-free DAGs (SF-\(k\))~\citep{barabasi1999emergence} with \(k \in \{2, 4\}\) and $d = 20$ vertices with expected number of edges $kd$. Edge weights are drawn uniformly from $\pm[0.5, 2]$, and i.i.d.\ noise from one of four distributions: Cauchy, Student-$t_3$, Laplace and mean-zero exponential. Columns are randomly permuted before estimation. We compare the top-down and bottom-up procedures of Algorithms~\ref{alg:td} and \ref{alg:bu} instantiated with the least-squares loss (the classical conditional-variance methods, denoted LS-TD and LS-BU) and with the loss matched to the noise ($\rho$-TD and $\rho$-BU): quantile regression for Laplace, iteratively reweighted least squares for Student-$t$ and Cauchy, and a linear program for the exponential. Scale parameters are treated as known, consistent with our known-$f$ setting, although some of the losses do not require them. We threshold at \(\omega = 0.3\), below the minimum true weight 0.5. As a baseline we include DirectLiNGAM \citep{shimizu2011directlingam, hyvarinen2013pairwise} using the Python package \texttt{lingam}, and \(\rho\)-DirectLiNGAM, which uses knowledge of the noise distribution for both, the ordering and the parent estimation steps.

To keep the comparisons readable, we restrict attention to the baselines most relevant to our setting and omit others for the following reasons. Continuous-optimization methods such as NOTEARS~\citep{zheng2018dags} and DAGMA~\citep{bello2022dagma} were consistently outperformed by LS-TD in our preliminary experiments. OptLiNGAM~\citep{oh2025optimal} is designed for sub-Gaussian tails, which none of our four noise distributions have. ETPIA~\citep{xie2019efficient} estimates entropies from moments up to fourth order, which do not exist for the Cauchy and Student-$t_3$ distributions. Finally, NPVAR~\citep{gao2020polynomial} and the method of \citet{rajendran2021structure} coincide with LS-TD in the linear setting, so including them would duplicate an existing curve.

We report the structural Hamming distance (SHD) to the true DAG, averaged over 50 replications, for $n \in \{50, 100, 200, 400, 800, 1600\}$. Figure~\ref{fig:er4_plot} reports the results for ER-4 random graphs. 

\begin{figure*}[htb]
    \centering
    \includegraphics[width=1\linewidth]{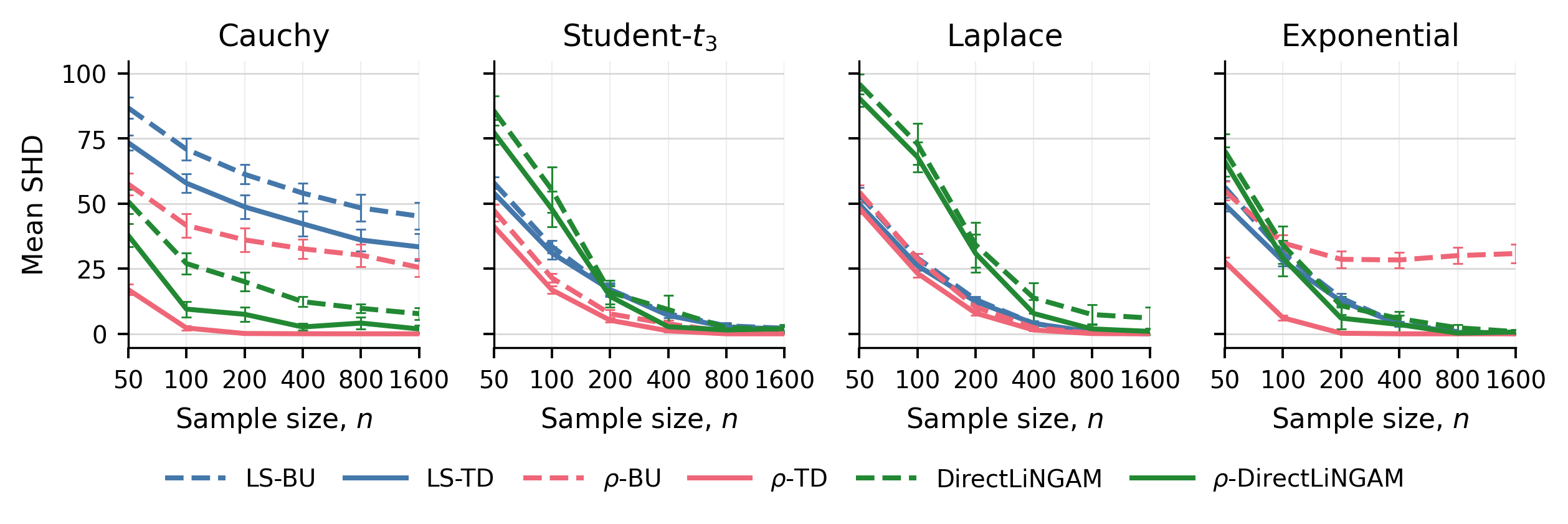}
    \caption{Structural Hamming distance between the true DAG and the learned DAG for different noise distributions, averaged over 50 replications of an ER-4 random graph. LS-BU and LS-TD are the existing least-squares bottom-up and top-down approaches. \(\rho\)-BU and \(\rho\)-TD refer to our matched-loss methods. \(\rho\)-DirectLiNGAM uses knowledge of the noise density for both steps of the method. Error bars represent 95\% confidence intervals.}
    \label{fig:er4_plot}
\end{figure*}

Within either direction, replacing the quadratic loss by the matched loss never hurts (when the regularity condition is satisfied) and helps increasingly as the noise moves towards heavier tails. The improvement is modest for Laplace noise, clear for $t_3$, and substantial for Cauchy. When the noise variance is undefined, the equal-variance premise of the classical methods is vacuous, and the least-squares procedures remain far from the truth (LS-BU still averages SHD $45$ under Cauchy noise at $n = 1600$). 

Even under noise satisfying (R1), the top-down variant is moderately but consistently ahead of the bottom-up variant at matched loss. This is a result of the order in which the two methods conduct their noisiest regressions as discussed in Section~\ref{ssec:separation}. In the two degenerate regimes of Cauchy or exponential noise, the structure recovered by $\rho$-BU remains far from the true DAG even at large sample-sizes, while $\rho$-TD recovers the graph rapidly.

DirectLiNGAM converges under all four settings, but is less sample-efficient than the matched top-down method throughout, especially at small $n$ (e.g., SHD $92$ versus $48$ for Laplace at $n = 50$, and $12$ versus $0$ for Cauchy at $n = 400$). Enhancing DirectLiNGAM by using knowledge of the noise distribution for ordering and parent estimation improves it uniformly but still lags behind \(\rho\)-TD. Results for sparser ($k = 2$) and scale-free graphs are qualitatively similar and deferred to Appendix~\ref{app:results}.

\section{Conclusion}
\label{sec:conclusion}

We revisited the ordering-based approach to learning linear additive noise SEMs and showed that the conditional variance at its core is one member of a family of conditional risks. The generalization revealed an asymmetry obscured by the Gaussian setting. The top-down direction, which asks only whether a residual is distributed as pure noise, is correct for every noise density of finite differential entropy. However, the bottom-up direction, which must detect that a vertex's noise is partially predictable from its children, requires the noise to meet the regularity condition, and can fail without it. The choice of direction, a matter of convenience under Gaussian noise, is a matter of correctness under non-Gaussian noise. Our experiments track this picture closely, and show that losses matched to the noise improve structure recovery over least squares with gains that grow with tail heaviness, to the point of restoring consistency where variance-based methods have none.

We see two main directions for future work. First, our theory concerns the population conditional risk. A finite-sample analysis is the natural next step, including sample complexity guarantees like those already established for variance-based ordering algorithms \citep{ghoshal2018learning, gao2022optimal}. Second, our methods assume that the noise density $f$ is known, whereas in practice it must be chosen or learned from the data. One part of this is to quantify how robust the ordering guarantees and population risk gaps are when $f$ is misspecified. A more ambitious direction is to learn $f$ together with the ordering, for example by fitting the residual density with flexible generative models. This would extend our methods to settings where the noise law is unknown.

\clearpage
\bibliography{references}

@book{Pearl_2009, place={Cambridge}, edition={Second}, title={Causality}, publisher={Cambridge University Press}, author={Pearl, Judea}, year={2009}}

@book{spirtes2000causation,
  title={Causation, prediction, and search},
  author={Spirtes, Peter and Glymour, Clark N and Scheines, Richard},
  year={2000},
  publisher={MIT Press}
}

@article{sachs2005causal,
  title={Causal protein-signaling networks derived from multiparameter single-cell data},
  author={Sachs, Karen and Perez, Omar and Pe'er, Dana and Lauffenburger, Douglas A and Nolan, Garry P},
  journal={Science},
  volume={308},
  number={5721},
  pages={523--529},
  year={2005},
  publisher={American Association for the Advancement of Science}
}

@article{sanford2012bayesian,
  title={A {Bayesian} network structure for operational risk modelling in structured finance operations},
  author={Sanford, Andrew D and Moosa, Imad A},
  journal={Journal of the Operational Research Society},
  volume={63},
  number={4},
  pages={431--444},
  year={2012},
  publisher={Taylor \& Francis}
}

@article{chickering2004large,
  title={Large-sample learning of {Bayesian} networks is {NP}-hard},
  author={Chickering, David Maxwell and Heckerman, David and Meek, Christopher},
  journal={Journal of Machine Learning Research},
  volume={5},
  number={Oct},
  pages={1287--1330},
  year={2004}
}

@article{spirtes1991algorithm,
  title={An algorithm for fast recovery of sparse causal graphs},
  author={Spirtes, Peter and Glymour, Clark},
  journal={Social Science Computer Review},
  volume={9},
  number={1},
  pages={62--72},
  year={1991},
  publisher={Sage Publications Sage CA: Thousand Oaks, CA}
}

@article{chickering2002optimal,
  title={Optimal structure identification with greedy search},
  author={Chickering, David Maxwell},
  journal={Journal of Machine Learning Research},
  volume={3},
  number={Nov},
  pages={507--554},
  year={2002}
}

@article{deng2024markov,
  title={Markov equivalence and consistency in differentiable structure learning},
  author={Deng, Chang and Bello, Kevin and Ravikumar, Pradeep and Aragam, Bryon},
  journal={Advances in Neural Information Processing Systems},
  volume={37},
  pages={91756--91797},
  year={2024}
}

@article{shimizu2006linear,
  title={A linear non-{Gaussian} acyclic model for causal discovery.},
  author={Shimizu, Shohei and Hoyer, Patrik O and Hyv{\"a}rinen, Aapo and Kerminen, Antti and Jordan, Michael},
  journal={Journal of Machine Learning Research},
  volume={7},
  number={10},
  year={2006}
}

@article{peters2014identifiability,
  title={Identifiability of {Gaussian} structural equation models with equal error variances},
  author={Peters, Jonas and B{\"u}hlmann, Peter},
  journal={Biometrika},
  volume={101},
  number={1},
  pages={219--228},
  year={2014},
  publisher={Oxford University Press}
}

@article{shimizu2011directlingam,
  title={{DirectLiNGAM}: A direct method for learning a linear non-{Gaussian} structural equation model},
  author={Shimizu, Shohei and Inazumi, Takanori and Sogawa, Yasuhiro and Hyv{\"a}rinen, Aapo and Kawahara, Yoshinobu and Washio, Takashi and Hoyer, Patrik O and Bollen, Kenneth},
  journal={Journal of Machine Learning Research-JMLR},
  volume={12},
  number={Apr},
  pages={1225--1248},
  year={2011}
}

@inproceedings{ghoshal2018learning,
  title={Learning linear structural equation models in polynomial time and sample complexity},
  author={Ghoshal, Asish and Honorio, Jean},
  booktitle={International Conference on Artificial Intelligence and Statistics},
  pages={1466--1475},
  year={2018},
  organization={PMLR}
}

@article{chen2019causal,
  title={On causal discovery with an equal-variance assumption},
  author={Chen, Wenyu and Drton, Mathias and Wang, Y Samuel},
  journal={Biometrika},
  volume={106},
  number={4},
  pages={973--980},
  year={2019},
  publisher={Oxford University Press}
}

@InProceedings{pmlr-v267-daskalakis25a,
  title = 	 {Learning {G}aussian {DAG} Models without Condition Number Bounds},
  author =       {Daskalakis, Constantinos Costis and Kandiros, Anthimos Vardis and Yao, Rui},
  booktitle = 	 {Proceedings of the 42nd International Conference on Machine Learning},
  pages = 	 {12554--12589},
  year = 	 {2025},
  volume = 	 {267},
}

@article{park2020heterogeneous,
  title={Identifiability of {Gaussian} linear structural equation models with homogeneous and heterogeneous error variances},
  author={Park, Gunwoong and Kim, Youngwhan},
  journal={Journal of the Korean Statistical Society},
  volume={49},
  number={1},
  pages={276--292},
  year={2020},
  publisher={Springer}
}

@article{zheng2018dags,
  title={{DAGS} with no tears: Continuous optimization for structure learning},
  author={Zheng, Xun and Aragam, Bryon and Ravikumar, Pradeep K and Xing, Eric P},
  journal={Advances in Neural Information Processing Systems},
  volume={31},
  year={2018}
}

@article{bello2022dagma,
  title={{DAGMA}: Learning {DAG}s via {M}-matrices and a Log-determinant acyclicity characterization},
  author={Bello, Kevin and Aragam, Bryon and Ravikumar, Pradeep},
  journal={Advances in Neural Information Processing Systems},
  volume={35},
  pages={8226--8239},
  year={2022}
}

@article{ng2020role,
  title={On the role of sparsity and {DAG} constraints for learning linear {DAGS}},
  author={Ng, Ignavier and Ghassami, AmirEmad and Zhang, Kun},
  journal={Advances in Neural Information Processing Systems},
  volume={33},
  pages={17943--17954},
  year={2020}
}

@article{hyvarinen2013pairwise,
  title={Pairwise likelihood ratios for estimation of non-{Gaussian} structural equation models},
  author={Hyv{\"a}rinen, Aapo and Smith, Stephen M},
  journal={The Journal of Machine Learning Research},
  volume={14},
  number={1},
  pages={111--152},
  year={2013},
  publisher={JMLR. org}
}

@article{park2023computationally,
  title={Computationally Efficient Learning of {Gaussian} Linear Structural Equation Models with Equal Error Variances},
  author={Park, Gunwoong},
  journal={Journal of Computational and Graphical Statistics},
  volume={32},
  number={3},
  pages={1060--1073},
  year={2023},
  publisher={Taylor \& Francis}
}

@inproceedings{gao2022optimal,
  title={Optimal estimation of {Gaussian} {DAG} models},
  author={Gao, Ming and Tai, Wai Ming and Aragam, Bryon},
  booktitle={International Conference on Artificial Intelligence and Statistics},
  pages={8738--8757},
  year={2022},
  organization={PMLR}
}

@article{eberhardt2017introduction,
  title={Introduction to the foundations of causal discovery},
  author={Eberhardt, Frederick},
  journal={International Journal of Data Science and Analytics},
  volume={3},
  number={2},
  pages={81--91},
  year={2017},
  publisher={Springer}
}

@article{erdos1960evolution,
  author  = {Erd{\H{o}}s, P. and R{\'e}nyi, A.},
  title   = {On the Evolution of Random Graphs},
  journal = {Publications of the Mathematical Institute of the Hungarian Academy of Sciences},
  volume  = {5},
  number  = {1},
  pages   = {17--61},
  year    = {1960}
}

@article{barabasi1999emergence,
  title={Emergence of scaling in random networks},
  author={Barab{\'a}si, Albert-L{\'a}szl{\'o} and Albert, R{\'e}ka},
  journal={Science},
  volume={286},
  number={5439},
  pages={509--512},
  year={1999},
  publisher={American Association for the Advancement of Science}
}

@article{shen2020challenges,
  title={Challenges and opportunities with causal discovery algorithms: application to Alzheimer’s pathophysiology},
  author={Shen, Xinpeng and Ma, Sisi and Vemuri, Prashanthi and Simon, Gyorgy},
  journal={Scientific Reports},
  volume={10},
  number={1},
  pages={2975},
  year={2020},
  publisher={Nature Publishing Group UK London}
}

@article{lefcheck2016piecewisesem,
  title={piecewiseSEM: Piecewise structural equation modelling in {R} for ecology, evolution, and systematics},
  author={Lefcheck, Jonathan S},
  journal={Methods in Ecology and Evolution},
  volume={7},
  number={5},
  pages={573--579},
  year={2016},
  publisher={Wiley Online Library}
}

@article{runge2019inferring,
  title={Inferring causation from time series in Earth system sciences},
  author={Runge, Jakob and Bathiany, Sebastian and Bollt, Erik and Camps-Valls, Gustau and Coumou, Dim and Deyle, Ethan and Glymour, Clark and Kretschmer, Marlene and Mahecha, Miguel D and Mu{\~n}oz-Mar{\'\i}, Jordi and others},
  journal={Nature Communications},
  volume={10},
  number={1},
  pages={2553},
  year={2019},
  publisher={Nature Publishing Group UK London}
}

@article{hyvarinen2010estimation,
  title={Estimation of a structural vector autoregression model using non-{Gaussianity}.},
  author={Hyv{\"a}rinen, Aapo and Zhang, Kun and Shimizu, Shohei and Hoyer, Patrik O},
  journal={Journal of Machine Learning Research},
  volume={11},
  number={5},
  year={2010}
}

@incollection{chickering1996learning,
  title={Learning {Bayesian} networks is {NP}-complete},
  author={Chickering, David Maxwell},
  booktitle={Learning from Data: Artificial Intelligence and Statistics V},
  pages={121--130},
  year={1996},
  publisher={Springer}
}

@article{reisach2021beware,
  title={Beware of the simulated {DAG}! causal discovery benchmarks may be easy to game},
  author={Reisach, Alexander and Seiler, Christof and Weichwald, Sebastian},
  journal={Advances in Neural Information Processing Systems},
  volume={34},
  pages={27772--27784},
  year={2021}
}

@inproceedings{ng2024structure,
  title={Structure learning with continuous optimization: A sober look and beyond},
  author={Ng, Ignavier and Huang, Biwei and Zhang, Kun},
  booktitle={Causal Learning and Reasoning},
  pages={71--105},
  year={2024},
  organization={PMLR}
}

@InProceedings{oh2025optimal,
  title = 	 {Optimal estimation of linear non-{Gaussian} structure equation models},
  author =       {Oh, Sunmin and Han, Seungsu and Park, Gunwoong},
  booktitle = 	 {Proceedings of The 28th International Conference on Artificial Intelligence and Statistics},
  pages = 	 {748--756},
  year = 	 {2025},
  volume = 	 {258},
  series = 	 {Proceedings of Machine Learning Research},
  month = 	 {03--05 May},
  publisher =    {PMLR},
}

@InProceedings{pmlr-v258-berrevoets25a,
  title = 	 {Differentiable Causal Structure Learning with Identifiability by {NOTIME}},
  author =       {Berrevoets, Jeroen and Raymaekers, Jakob and van der Schaar, Mihaela and Verdonck, Tim and Yao, Ruicong},
  booktitle = 	 {Proceedings of The 28th International Conference on Artificial Intelligence and Statistics},
  pages = 	 {3115--3123},
  year = 	 {2025},
  volume = 	 {258},
  series = 	 {Proceedings of Machine Learning Research},
  month = 	 {03--05 May},
  publisher =    {PMLR},
}

@article{gao2020polynomial,
  title={A polynomial-time algorithm for learning nonparametric causal graphs},
  author={Gao, Ming and Ding, Yi and Aragam, Bryon},
  journal={Advances in Neural Information Processing Systems},
  volume={33},
  pages={11599--11611},
  year={2020}
}

@article{rajendran2021structure,
  title={Structure learning in polynomial time: Greedy algorithms, {Bregman} information, and exponential families},
  author={Rajendran, Goutham and Kivva, Bohdan and Gao, Ming and Aragam, Bryon},
  journal={Advances in Neural Information Processing Systems},
  volume={34},
  pages={18660--18672},
  year={2021}
}

@article{xie2019efficient,
  title={An efficient entropy-based causal discovery method for linear structural equation models with IID noise variables},
  author={Xie, Feng and Cai, Ruichu and Zeng, Yan and Gao, Jiantao and Hao, Zhifeng},
  journal={IEEE Transactions on Neural Networks and Learning Systems},
  volume={31},
  number={5},
  pages={1667--1680},
  year={2019},
  publisher={IEEE}
}

@inproceedings{feigenbaum2024causal,
  title={Causal layering via conditional entropy},
  author={Feigenbaum, Itai and Arpit, Devansh and Heinecke, Shelby and Niebles, Juan Carlos and Yao, Weiran and Xiong, Caiming and Savarese, Silvio and Wang, Huan},
  booktitle={Causal Learning and Reasoning},
  pages={1176--1191},
  year={2024},
  organization={PMLR}
}

@article{loh2014high,
  title={High-dimensional learning of linear causal networks via inverse covariance estimation},
  author={Loh, Po-Ling and B{\"u}hlmann, Peter},
  journal={Journal of Machine Learning Research},
  volume={15},
  number={140},
  pages={3065--3105},
  year={2014}
}

\clearpage
\onecolumn

 \appendix
\section{Proofs}\label{app:proofs}

\subsection{Proof of Lemma~\ref{lem:ls_sink}}
\textit{\noindent Write the linear structural equation model as~\eqref{eq:sem_linear} where \(\Lambda_{ji}\neq 0\) precisely when \(i\to j\in E\). Since the graph is acyclic, \(I-\Lambda\) is invertible, and hence \( X=(\bm I-\bm \Lambda)^{-1}N.\)
Therefore, the covariance matrix of \(X\) is
\[
\bm \Sigma = \mathbb{E}\left[ XX^\top\right] = \mathbb{E}\left[ (\bm I - \bm \Lambda)^{-1}NN^\top(\bm I - \bm \Lambda)^{-\top}\right]
=\sigma^2(\bm I-\bm \Lambda)^{-1} (\bm I-\bm \Lambda)^{-\top},
\]
and the precision matrix is
\[
\bm \Omega=\bm \Sigma^{-1} 
=\frac{1}{\sigma^2}(\bm I-\bm \Lambda)^{\top} (\bm I-\bm \Lambda).
\]
For any random vector with covariance $\bm \Sigma$, the minimal mean squared error of linear prediction for one coordinate given all the others is the reciprocal of the corresponding diagonal entry of the precision matrix. Thus, \( \operatorname{Var}(X_i\mid X_{-i})=\frac{1}{\Omega_{ii}}. \)
Moreover,
\[
\Omega_{ii}
=\frac{1}{\sigma^2}
\sum_{j=1}^d (I-\Lambda)_{ji}^2.
\]
Because \(\Lambda_{ii}=0\), the \(j=i\) term equals \(1\). For \(j\neq i\), the entry \(\Lambda_{ji}\) is nonzero exactly when \(i\to j\), that is, when \(j\in\mathrm{CH}(i)\). Consequently,
\[
\Omega_{ii}
=\frac{1}{\sigma^2}
\left(1+\sum_{j\in\mathrm{CH}(i)}\Lambda_{ji}^2\right),
\]
and therefore
\[
\operatorname{Var}(X_i\mid X_{-i})
=\frac{\sigma^2}
{1+\sum_{j\in\mathrm{CH}(i)}\Lambda_{ji}^2}.
\]}

\subsection{Proof of Lemma~\ref{lem:certificate}}
\textit{\noindent Since every variable in $S$ is a non-descendant of $i$, the noise $N_i$ is independent of $X_S$, and $X_i - \beta^\top X_S = N_i + Z_\beta$ with $Z_\beta = \sum_{p \in \mathrm{PA}(i)} \Lambda_{ip} X_p - \beta^\top X_S$. The expected loss of any random variable $Y$ under $\rho_f = -\log f$ is the cross-entropy between the distributions of $Y$, which we denote as \(g_Y\) here,  and $f$, which decomposes as $h(g_Y) + \mathrm{KL}(g_Y \| f)$. The KL term is nonnegative, and $h(g_Y) \ge h(f)$ because differential entropy is nondecreasing under the addition of independent noise, with equality if and only if $Z_\beta$ is zero. If $\mathrm{PA}(i) \subseteq S$, the true weights give $Z_\beta = 0$ and the risk $c_f$, which by the above cannot be improved.}

\textit{Now let $S$ be ancestral with $\mathrm{PA}(i) \not\subseteq S$, and let $p$ be the last vertex of $\mathrm{PA}(i) \setminus S$ in a topological order. No other missing parent is a descendant of $p$, and since $S$ is ancestral and $p \notin S$, no vertex of $S$ is either. Hence $N_p$ enters $Z_\beta$ only through $\Lambda_{ip} X_p$, with coefficient $\Lambda_{ip} \neq 0$, for every $\beta$. Conditioning on all noises other than $N_i$ and $N_p$ gives $h(N_i + Z_\beta) \ge h(N_i + \Lambda_{ip} N_p) > h(f)$, the last step by the entropy power inequality. This bound does not depend on $\beta$, so $\mathcal{R}_f(X_i \mid X_S) > c_f$. This argument does not use faithfulness.
}

\subsection{Proof of Theorem~\ref{thm:top-down}}
\textit{\noindent We argue by induction that each $S_k$ is ancestral and each selected $t_k$ has all parents in $S_k$. The empty set is ancestral. Given ancestral $S_k$, some remaining variable has all parents in $S_k$ (take a remaining variable minimal in a fixed topological order: its parents precede it, and every remaining variable succeeds it, so its parents lie in $S_k$); by Lemma~\ref{lem:certificate} such variables attain risk exactly $c_f$, while any remaining variable with a missing parent has risk strictly greater. The argmin is therefore an eligible variable, and $S_{k+1}$ remains ancestral.
}

\subsection{Proof of Lemma~\ref{lem:sink-removal}}
\textit{\noindent Since $i$ is a sink, $X_i$ appears in no structural equation other than its own; deleting that equation leaves the remaining equations, and hence the joint distribution of $X_{-i}$, unchanged.
}

\subsection{Proof of Lemma~\ref{lem:strict-gap}}
\textit{\noindent Let $X_c$ be a child of $X_j$ with $b = \Lambda_{cj} \neq 0$. The linear combination
\begin{equation}
    W \;=\; X_c \;-\!\! \sum_{k \in \mathrm{PA}(c) \setminus \{j\}} \!\! \Lambda_{ck} X_k \;-\; b\!\! \sum_{k \in \mathrm{PA}(j)} \!\! \Lambda_{jk} X_k
    \;=\; b N_j + N_c
\end{equation}
involves only covariates in $X_{-j}$ and therefore lies in their linear span. Let $\beta_0$ be the feasible point realizing \eqref{eq:free-upper-bound}, with residual $N_j$, and perturb along $W$: the residual of $\beta_0 + \varepsilon\, w$ (where $w$ collects the coefficients of $W$) is $(1 - \varepsilon b) N_j - \varepsilon N_c$. By (R1), the risk $\varphi(\varepsilon)$ along this path satisfies $\varphi(0) = c_f$ and
\begin{equation}
    \varphi'(0) \;=\; -\,\mathbb{E}\big[\psi_f(N_j)\,(b N_j + N_c)\big]
    \;=\; -\,b\;\mathbb{E}\big[N_j\psi_f(N_j)\big] \;-\; \mathbb{E}\,\big[\psi_f(N_j)\big]\;\mathbb{E}\,\big[N_c\big]
    \;=\; -\,b,
\end{equation}
where the expectation splits by independence of $N_j$ and $N_c$ and integrability of $N_c$, and the last equality uses the identities in (R1). Since $b \neq 0$, a strict descent direction exists, and the minimal risk is strictly below $\varphi(0) = c_f$.
}

\subsection{Proof of Theorem~\ref{thm:bottom-up}}
\textit{\noindent By Lemma~\ref{lem:sink-removal}, at every step the remaining variables follow a linear SEM with i.i.d.\ noise $\sim f$ whose graph is the induced subgraph, so it suffices to argue one step. By Lemma~\ref{lem:certificate}, every sink of the current graph attains risk exactly $c_f$; by Lemma~\ref{lem:strict-gap}, every non-sink has risk strictly below $c_f$. Hence the argmax is a sink.
}

\section{Failure mechanisms for the bottom-up statistic}\label{app:failure}

This appendix gives the population calculations behind the exponential and Cauchy examples discussed in Section~\ref{ssec:separation}.

\subsection{Boundary-dominated noise}

\begin{proposition}[Exponential noise]Let\[f(u)=e^{-(u+1)}\mathbf{1}\{u\ge-1\},\]the density of a mean-zero unit exponential random variable. Then, for every vertex \(j\),\[\mathcal{R}_f(X_j\mid X_{-j})=c_f=1.\]Consequently, the full-conditional risk contains no information about which vertices are sinks. By contrast, if \(S\) is ancestral and \(\operatorname{PA}(i)\not\subseteq S\), then\[\mathcal{R}_f(X_i\mid X_S)=+\infty.\]\end{proposition}

\begin{proof}The matched-loss is \[\rho_f(u)=\begin{cases} u+1, & u\geq-1,\\ +\infty, & u<-1.\end{cases}\]Since \(\mathbb EN=0\),\[c_f=\mathbb E(N+1)=1.\]

For any vertex \(j\), the coefficient vector placing the true structural coefficients on \(\operatorname{PA}(j)\) and zero elsewhere is feasible in the regression on \(X_{-j}\), and its residual is \(N_j\). Hence\[\mathcal{R}_f(X_j\mid X_{-j})\leq1.\]Conversely, every residual \(R=X_j-\beta^\top X_{-j}\) is centered. If \(R<-1\) with positive probability, then its risk is infinite. If \(R\geq-1\), then\[\mathbb E\rho(R)=\mathbb E(R+1)=1.\]Thus no coefficient vector achieves risk below one, proving\[\mathcal{R}_f(X_j\mid X_{-j})=1\]for every \(j\).

For the top-down claim, Lemma~\ref{lem:certificate} writes any residual as\[X_i-\beta^\top X_S=N_i+Z_\beta,\]where \(Z_\beta\) is centered and independent of \(N_i\). Finite risk would require\[N_i+Z_\beta\geq-1\qquad\text{almost surely}.\]Because \(N_i\) takes values arbitrarily close to \(-1\), independence forces \(Z_\beta\geq0\) almost surely. Since \(Z_\beta\) is centered, this implies \(Z_\beta=0\) almost surely. When \(S\) is ancestral but omits a parent of \(i\), Lemma~\ref{lem:certificate} shows that \(Z_\beta=0\) is impossible. Therefore every candidate coefficient vector has infinite risk.\end{proof}

The failure occurs because the location information is carried by the boundary of the support, while the loss is affine in the interior. The bottom-up statistic minimizes risk and therefore cannot turn infeasibility away from the true coefficient vector into a value below \(c_f\).

\subsection{Scale-additive noise}

\begin{proposition}[Cauchy noise]Let \(N_1,N_2\) be independent standard Cauchy variables and consider\[X_1=N_1,\qquad X_2=bX_1+N_2.\]If \(0<|b|\leq1\), then\[\mathcal{R}_f(X_1\mid X_2)=\mathcal{R}_f(X_2\mid X_1)=c_f=\log(4\pi),\]so the bottom-up statistic cannot distinguish the sink. If \(|b|>1\), then\[\mathcal{R}_f(X_1\mid X_2)=2\log\left(1+\frac1{|b|}\right)+\log\pi<c_f.\]The top-down statistic identifies the source for every \(b\neq0\).
\end{proposition}

\begin{proof}For the standard Cauchy density,\[\rho_f(u)=\log\pi+\log(1+u^2),\qquad c_f=\log(4\pi).\]A Cauchy variable with scale \(s\) has matched risk \[2\log(1+s)+\log\pi,\]which is strictly increasing in \(s\). Moreover, independent Cauchy variables combine by addition of scales:\[aN_1+cN_2\sim\operatorname{Cauchy}(|a|+|c|).\]

For the sink \(X_2\), the residual after regression on \(X_1\) is\[X_2-tX_1=(b-t)N_1+N_2,\]whose scale \(|b-t|+1\) is minimized at \(t=b\). The minimum scale is one, so\[\mathcal{R}_f(X_2\mid X_1)=c_f.\]

For the non-sink \(X_1\),\[X_1-tX_2=(1-bt)N_1-tN_2,\]whose scale is\[s(t)=|1-bt|+|t|.\]If \(|b|\leq1\), then\[|1-bt|+|t|\geq |1-bt|+|bt|\geq1,\]
and equality is attained at \(t=0\). Thus the non-sink also has risk \(c_f\). If \(|b|>1\), the scale is minimized at \(t=1/b\), where it equals \(1/|b|\), giving the stated risk.

For top-down, \(X_1\) has Cauchy scale one, whereas \(X_2=bN_1+N_2\) has scale \(1+|b|\). Since the matched risk increases strictly with scale, \(X_1\) is selected as the source whenever \(b\neq0\).\end{proof}

\section{Experimental Results for Random Graphs}
\label{app:results}

\paragraph{Computing infrastructure.} All experiments were run on a cloud virtual machine with 4 CPU cores (AMD EPYC 7B13), 16 GB of RAM and no GPU, using Python 3.10 (package versions are listed in the supplementary code's \verb|requirements.txt|). Running all experiments takes about 5–6 hours with the replications parallelized across the 4 cores.

\paragraph{Results.} The experiments in Section~\ref{sec:experiments} use Erd\H{o}s--R\'enyi graphs with $4d$ expected edges (ER-4). This appendix repeats them on ER-2, SF-2 and SF-4 graphs, shown in Figures~\ref{fig:er2_plot}--\ref{fig:sf4_plot}. Scale-free graphs are generated by preferential attachment~\citep{barabasi1999emergence}, which produces a few hub vertices with many neighbors, in contrast to the more homogeneous degree distribution of Erd\H{o}s--R\'enyi graphs. Everything other than the graph family is held fixed as in the ER-4 experiments: $d = 20$ vertices, edge weights drawn uniformly from $\pm[0.5, 2]$, threshold $\omega = 0.3$, and $50$ replications at each sample size.

The conclusions of Section~\ref{sec:experiments} carry over to all three graph families. The matched-loss top-down method $\rho$-TD attains the lowest or near-lowest SHD in every panel, and its advantage is largest at small sample sizes and under heavy-tailed noise. For Laplace and Student-$t_3$ noise, all methods perform similarly by $n = 400$, while DirectLiNGAM remains noticeably worse at small $n$ on the denser SF-4 graphs. As with ER-4 experiments, using information about the noise distribution in $\rho$-DirectLiNGAM improves performance compared with DirectLiNGAM.

Both degenerate regimes of Section~\ref{ssec:separation} reappear. Under exponential noise, $\rho$-BU plateaus at an SHD that does not decrease with $n$, whereas $\rho$-TD recovers the graph almost exactly from $n = 200$ onward. Under Cauchy noise, $\rho$-BU improves on the least-squares methods but stays clearly above $\rho$-TD. The least-squares methods LS-TD and LS-BU decrease only slowly and remain far from the truth at $n = 1600$, since the variance on which least squares relies is undefined for Cauchy noise. The performance ranking of the methods remains the same as in the ER-4 experiments.

\begin{figure}[!ht]
    \centering
    \includegraphics[width=1\linewidth]{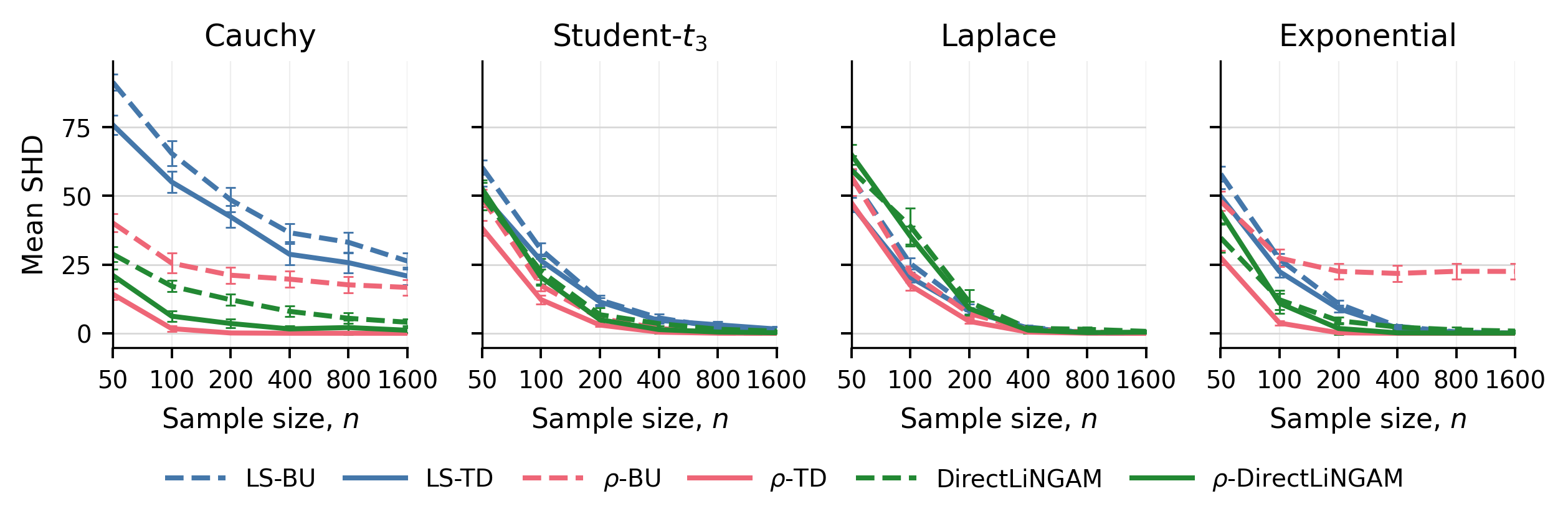}
        \caption{Structural Hamming distance between the true DAG and the learned DAG for different noise distributions, averaged over 50 replications of the ER-2 random graph. Error bars represent 95\% confidence intervals.}
    \label{fig:er2_plot}
\end{figure}

\begin{figure}[!ht]
    \centering
    \includegraphics[width=1\linewidth]{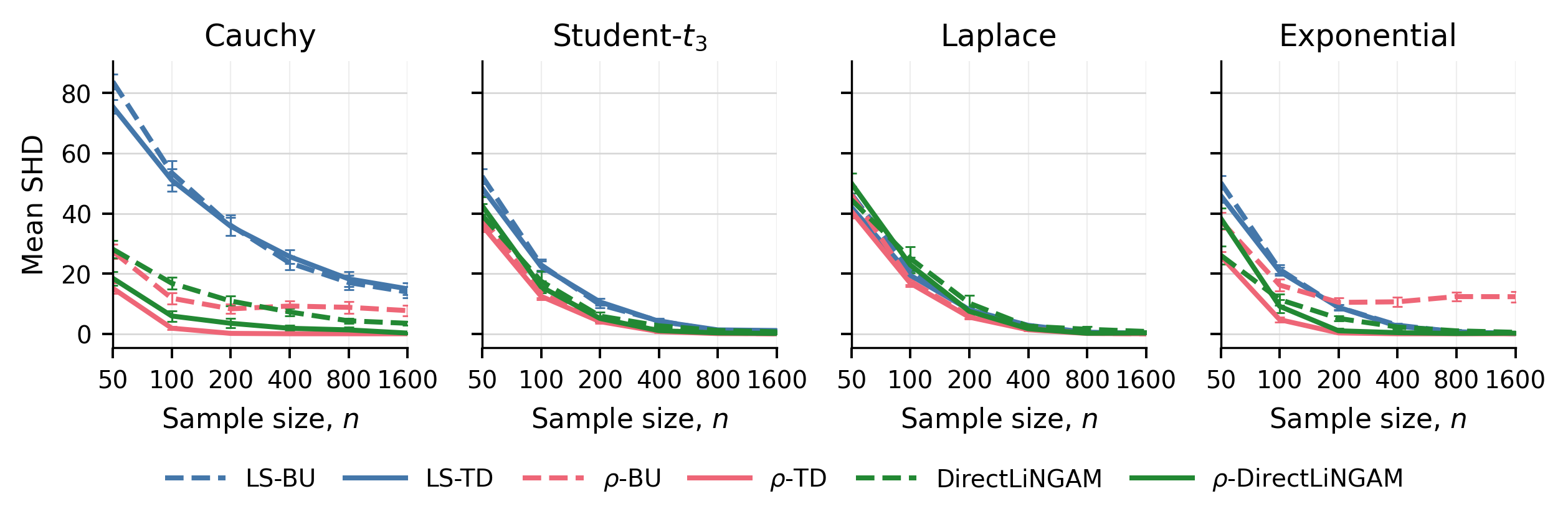}
    \caption{Structural Hamming distance between the true DAG and the learned DAG for different noise distributions, averaged over 50 replications of the SF-2 random graph. Error bars represent 95\% confidence intervals.}
    \label{fig:sf2_plot}
\end{figure}

\begin{figure}[!ht]
    \centering
    \includegraphics[width=1\linewidth]{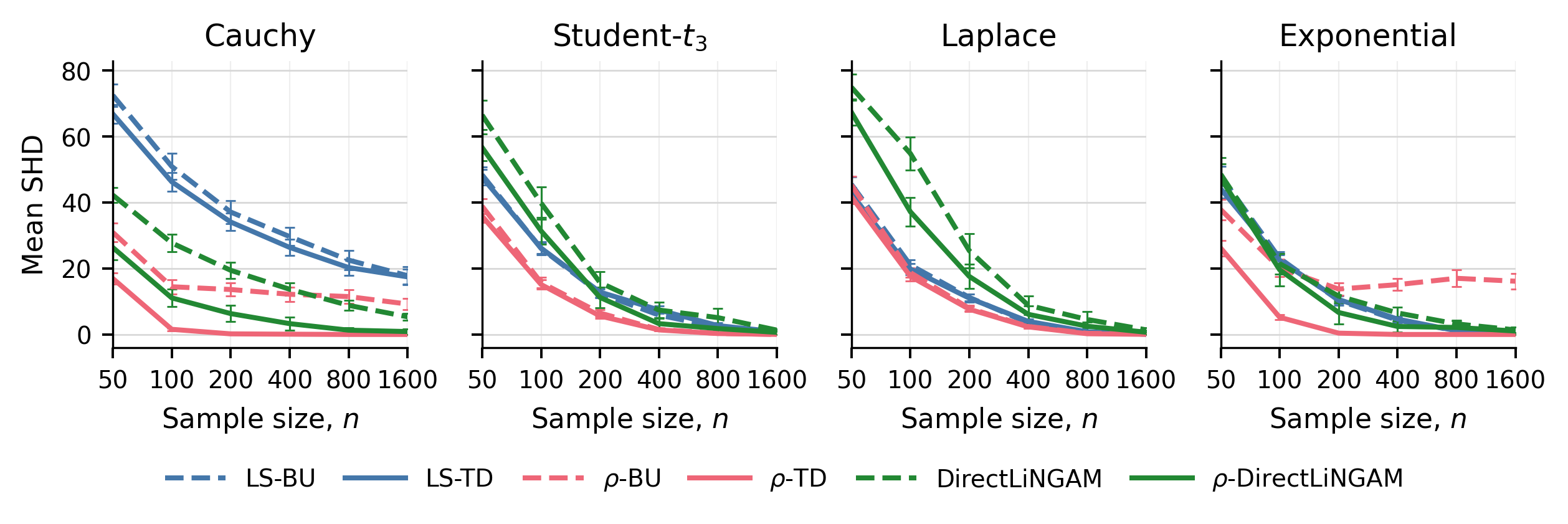}
    \caption{Structural Hamming distance between the true DAG and the learned DAG for different noise distributions, averaged over 50 replications of the SF-4 random graph. Error bars represent 95\% confidence intervals.}
    \label{fig:sf4_plot}
\end{figure}
\end{document}